\documentclass[runningheads]{llncs}
\usepackage{graphicx}
\usepackage{comment}
\usepackage{amsmath, amsfonts}
\usepackage{booktabs}

\usepackage{caption}
\usepackage{subcaption}
\DeclareMathOperator*{\argmax}{arg\,max}

\usepackage{algorithm}
\usepackage[noend]{algpseudocode}
\begin{document}
\title{Knowledge Infused Policy Gradients with Upper Confidence Bound for Relational Bandits}

\author{Kaushik Roy\inst{1} \and
Qi Zhang\inst{1} \and
Manas Gaur\inst{1}\and
Amit Sheth\inst{1}}
\authorrunning{F. Author et al.}
%
\institute{
\email{kaushikr@email.sc.edu qz5@cse.sc.edu mgaur@email.sc.edu amit@sc.edu}
Artificial Intelligence Institute, University of South Carolina, Columbia, USA\\
}

\maketitle
\begin{abstract}
 Contextual Bandits find important use cases in various real-life scenarios such as online advertising, recommendation systems, healthcare, etc. However, most of the algorithms use flat feature vectors to represent context whereas, in the real world, there is a varying number of objects and relations among them to model in the context. For example, in a music recommendation system, the user context contains what music they listen to, which artists create this music, the artist albums, etc. Adding richer relational context representations also introduces a much larger context space making exploration-exploitation harder. To improve the efficiency of exploration-exploitation knowledge about the context can be infused to guide the exploration-exploitation strategy. Relational context representations allow a natural way for humans to specify knowledge owing to their descriptive nature. We propose an adaptation of Knowledge Infused Policy Gradients to the Contextual Bandit setting and a novel Knowledge Infused Policy Gradients Upper Confidence Bound algorithm and perform an experimental analysis of a simulated music recommendation dataset and various real-life datasets where expert knowledge can drastically reduce the total regret and where it cannot.
\end{abstract}

\section{Introduction}
Contextual Bandits (CB) are an extension of the classical Multi-Armed-Bandits (MAB) setting where the arm choice depends also on a specific context \cite{langford2007epoch}. 
As an example, in a music recommendation system, the choice of song recommendation (the arm choice) depends on the user context (user preferences concerning genre, artists, etc). In the real world, the context is often multi-relational but most CB algorithms do not model multi-relational context and instead use flat feature vectors that contain attribute-value pairs \cite{zhou2015survey}. While relational modeling allows us to enrich user context, it further complicates the exploration-exploitation problem due to the introduction of a much larger context space. Initially, when much of the space of context-arm configurations are unexplored, aggressive exploitation may yield sub-optimal total regret. Hence, a principled exploration-exploitation strategy that encodes high uncertainty initially that tapers off with more information is required to effectively achieve near-optimal total regret. The Upper-Confidence-Bound (UCB) algorithm uses an additional term to model initial uncertainty that tapers off during each arm pull \cite{lai1985asymptotically}. However, though the UCB provides a reasonable generalized heuristic, the exploration strategy can further be improved with more information about the reward distribution,
for example, if it is known that the expected reward follows a Gaussian distribution. This is what Thompson Sampling does - incorporates a prior distribution over the expected rewards for each arm and updates a Bayesian posterior \cite{thompson1933likelihood}. If external knowledge is available the posterior can be reshaped with knowledge infusion \cite{chapelle2011empirical}. An example of this knowledge for the IMDB dataset described in Section \ref{sec:exp} can be seen in Figure \ref{fig:knowledge} and the detailed formulation for the knowledge used is described in Section \ref{sec:four}.
A couple of issues arise with posterior reshaping: a) The choice of reshaping function is difficult to determine in a principled manner, and b) The form of the prior and posterior is usually chosen to exploit a likelihood-conjugate before analytically compute posterior estimates as sampling is typically inefficient. Similarly, the choice of reshaping function needs to either be amenable to efficient sampling for exploration or analytically computed. Thus, we observe that we can instead directly optimize for the optimal arm choice through policy gradient methods \cite{peters2010policy}. Using a Bayesian formulation for optimization of policy in functional space, we can see that the knowledge infused reshape function can be automatically learned by an adaption of the \emph{Knowledge Infused Policy Gradients} (KIPG) algorithm for the Reinforcement Learning (RL) setting to the CB setting \cite{roy2021knowledge}, which takes as input a state and knowledge, and outputs an action.\\ 
\begin{figure}[!h]
    \centering
    \includegraphics[width=\textwidth]{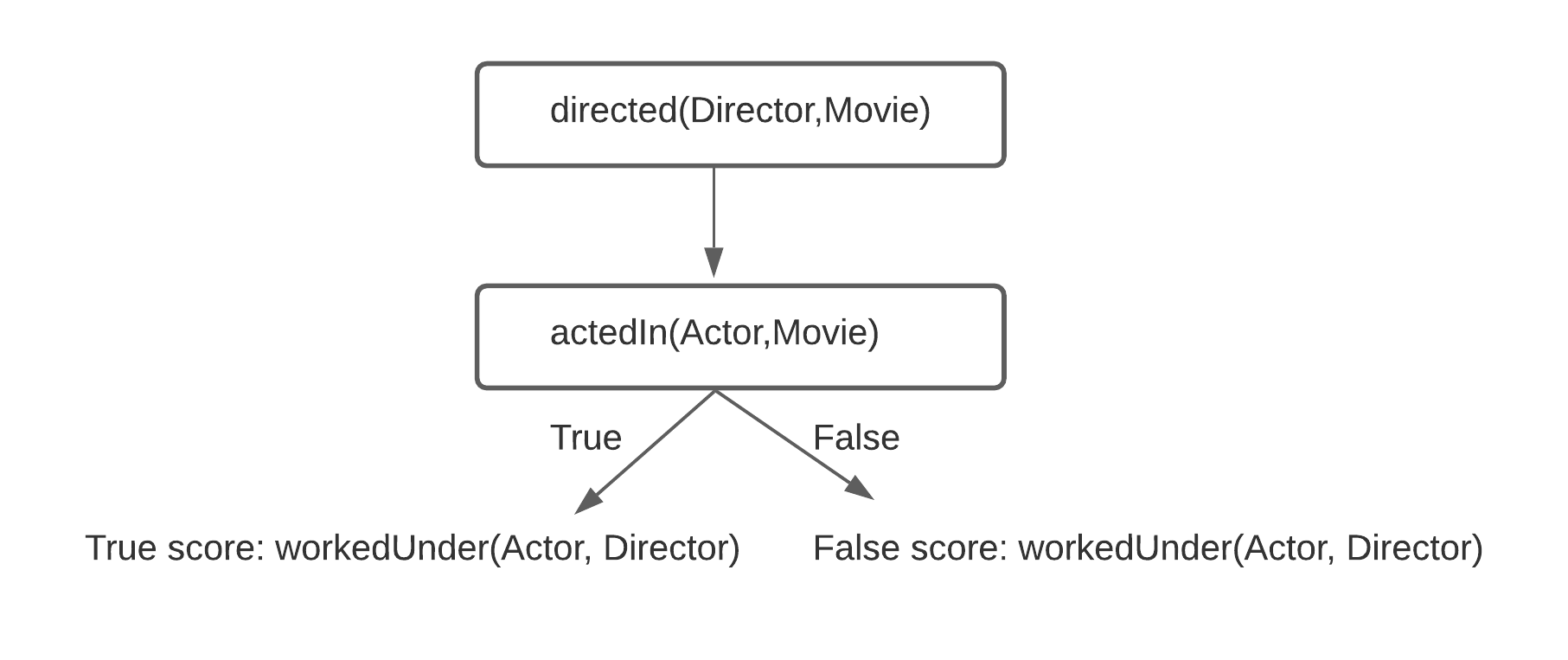}
    \caption{Example of expert knowledge in the IMDB dataset. This says that if a director directed a movie in which an actor acted, there is a chance that the actor worked under the director.}
    \label{fig:knowledge}
\end{figure}
\indent The CB setting presents a unique challenge for knowledge infusion. Since arm pulling happens in an online fashion, the human knowledge about the user is uncertain until the human observes some arm choices. First, we adapt the KIPG algorithm from the RL to the CB setting and then we improve upon it to make it less aggressive in its knowledge infusion strategy when the human is still uncertain about the user's preferences. 
For this reason, we develop a UCB style uncertainty measure that considers the initial uncertainty as the human gathers more information about the user context, before providing knowledge. Thus, we develop a  \emph{Knowledge Infused Policy Gradient Upper Confidence Bound} (KIPGUCB) algorithm to incorporate human uncertainty in providing knowledge in the knowledge infusion strategy.
Our methodological contributions are as follows:
\begin{itemize}
    \item We adapt KIPG for the RL setting to the CB setting to reduce the total regret with high-quality knowledge.
    \item We develop a novel relational CB algorithm KIPGUCB that reduces regret through knowledge infusion with both high-quality and noisy knowledge using exploration.
    \item Theoretically, we observe that KIPG is fundamentally a gradient ascent method and derive a regret bound that depends on the knowledge. We also derive a confidence bound for when the knowledge is noisy.
    \item Empirically, through experiments on various real-life datasets, we perform analysis of settings where KIPGUCB achieves a drastic reduction in total regret. We compare KIPGUCB to KIPG without a confidence bound and compare against the Relational Boosted Bandits algorithm (RB2) \cite{kakadiya2020relational}, a state-of-the-art contextual bandit algorithm for relational domains.
\end{itemize}

\section{Problem Setting}
We consider the problem setting of Bernoulli Contextual Bandits with relational features.
Formally, at each step $k$, when an arm $i\in[N]:=\{1,2,...,N\}$ is pulled from among $N$ arms, the reward $r_{k}(i) \in \{0,1\}$ is Bernoulli. Also, pulling an arm $i$ depends on a relational context $c_{k}(i)$. Since $\pi_{k}(i)$, which represents the probability of choosing arm $i$ given context $c_{k}(i)$, is expected to be high if $P(r_{k}(i)=1|c(i))$ is high, we directly maximize the total reward over $K$ arm choices, $\sum_{k=1}^{K}\pi_{k}(i)r_{k}(i)$.
Here $\pi_{k}(i)=\sigma(\Psi_{k}(i))$, and $\sigma$ is the sigmoid function. $\Psi_{k}(i)$ is a relational function that includes the relational context $c_{k}(i)$.

\section{Knowledge Infused Policy Gradients}
In this section, we develop the formulation for the KIPG adaptation to the CB setting.
We first describe policy gradients for CB, extend it to functional spaces and then use Bayes rule to derive the KIPG formulation. 
In next section, we show the connection of KIPG to Thompson Sampling with posterior reshaping and the Exponential Weight for Exploration and Exploitation (Exp3) algorithm \cite{auer2002nonstochastic}, which is also derived from a gradient ascent procedure (mirror ascent) that can be seen as an instance of KIPG. 

\subsubsection{Policy Gradients for Contextual Bandits with Flat Feature Vectors}
In policy gradient methods the probability of picking an arm $i$ given context $c(i)$, is parameterized as $\pi(i) = \sigma(\theta(i)^{T}c(i))$. 
\sloppy We want to maximize the expected reward over $K$ arm pulls $\sum_{k=1}^{K}\pi_{k}(i)r_{k}(i)$. We update the parameters for arm $i$, at each $k+1$, using gradient ascent as 
$\theta_{k+1}(i) = \theta_{k}(i) + \eta\nabla_{\theta_{k}(i)}(\sum_{k}\pi_{k}(i)r_{k}(i))$. 
Here we note that the gradient $\nabla_{\theta_{k}(i)}\pi_{k}(i) = \pi_{k}(i)\nabla_{\theta_{k}(i)}\log(\pi_{k}(i))$ and thus we optimize:
\[\theta_{k+1}(i) = \theta_{k}(i) + \eta(\sum_{k}\pi_{k}(i)\nabla_{\theta_{k}(i)}\log(\pi_{k}(i))r_{k}(i))\]
\subsubsection{Policy Gradients for Contextual Bandits in Functional Space}
In functional space the $\theta(i)^{T}c(i)$ is replaced by a function $\Psi(i)$ i.e. $\pi(i) = \sigma(\Psi(i))$, where $\Psi(i)$ is a relational function that includes context $c(i)$. Thus, the policy gradient update becomes
\[\Psi_{k}(i) = \Psi_{k}(i) + \eta(\sum_{k}\pi_{k}(i)\nabla_{\Psi_{k}(i)}\log(\pi_{k}(i))r_{k}(i)).\]
Here, $\Psi_{k}(i)$ at each iteration of policy gradients is grown stage wise. We start with a $\Psi_{0}(i)$ and update $\Psi_{K}(i) = \Psi_{0}(i) +  \sum_{k=1}^{K}\eta\delta_{k}(i)$, where each $\delta_{k}(i)$ fits a function to $\pi_{k}(i)\nabla_{\Psi_{k}(i)}\log(\pi_{k}(i))r_{k}(i)$ \cite{kersting2008non}. In our experiments this function is a TILDE regression tree \cite{blockeel1998top}.
However, we derive a Bayesian formulation for $\pi_{k}(i)$ for knowledge infusion. Thus, After pulling an arm $i$ at step $k$, and observing rewards $r_{k}(i)$, and context $c_{k}(i)$, using Bayes rule we can write \[P(\Psi_{k}(i)|r_{k}(i)) = \frac{P(r_{k}(i)|\Psi_{k}(i))P(\Psi_{k}(i))}{\int_{\Psi_{k}(i)}P(r_{k}(i)|\Psi_{k}(i))P(\Psi_{k}(i))}.\]
Using the sigmoid function we can set $P(r_{k}(i)|\Psi_{k}(i)) = \sigma(\Psi_{k}(i)) = \frac{e^{\Psi_{k}(i)}}{(1+e^{\Psi_{k}(i)})}$ and use the Bayesian posterior to obtain a prior informed policy as
\[\pi_{k}(i) = \frac{\sigma(\Psi_{k}(i))P(\Psi_{k}(i))}{\int_{\Psi_{k}(i)}\sigma(\Psi_{k}(i))P(\Psi_{k}(i))}.\]
\sloppy To optimize using policy gradients, again we note that $\nabla_{\Psi_{k}(i)}(\pi_{k}(i)) = \pi_{k}(i)\nabla_{\Psi_{k}(i)}\log(\pi_{k}(i))$
If we use a form for $P(\Psi_{k}(i))$, for which the normalization doesn't depend on $\Psi_{k}(i)$ such as a Laplace or a Gaussian distribution, we can take the log on both sides without loss of generality to derive the gradient $\nabla_{\Psi_{k}(i)}\log(\pi_{k}(i))$:
\[\log(\pi_{k}(i)) \propto \log(\sigma(\Psi_{k}(i))) + \log(P(\Psi_{k}(i))),\] taking the gradient gives us
\[(I_{k}(i) - \sigma(\Psi_{k}(i))) +\nabla_{\Psi_{k}(i)}\log(P(\Psi_{k}(i))),\] where $I_{k}(i)$ is the indicator function representing if arm $i$ was chosen at step $k$.
Now we can employ functional gradient ascent by fitting a weak learner (such as a TILDE tree for relational context, or linear function for propositional context) to the gradient $\pi_{k}(i)\nabla_{\Psi_{k}(i)}\log(\pi_{k}(i))$. Note here that $\log(P(\Psi_{k}(i)))$ will determine the nature of knowledge infused into the policy gradient learning setup at each $k$. We call this approach Knowledge Infused Policy Gradients (KIPG).
\section{Formulation of Knowledge Infusion}{\label{sec:four}}
 At each $k$, the prior over functions $\Psi_{k}(i)$ for each arm $P(\Psi_{k}(i))$ determines the knowledge infusion process. We now show the formulation for infusing arm preferences as knowledge as we use this in our experiments. Depending on the problem needs, the user may pick their choice of $P(\Psi_{k}(i))$ to be any distribution. Since our knowledge is given as weighted preferences over arm choices, we will cover two intuitive ways to formulate the knowledge and derive the formulation we use in our experiments.
 
 \paragraph{$P(\Psi_{k}(i)) = Normal(\mu,\Sigma)$: }Given a context included in $\Psi_{k}(i)$, if we want to prefer the arm choice $i$, we can specify this knowledge using a two step procedure. 
 First we set $\Psi_{k}(i)_{knowledge} = \alpha$, where $\alpha \geq 1$. Then we set $P(\Psi_{k}(i))  = Normal(\mu = \Psi_{k}(i)_{knowledge}-\sigma(\Psi_{k}(i)),\Sigma=I)$. Similarly if the arm choice $i$ is not preferred, $\Psi_{k}(i)_{knowledge} = -\alpha$. Here $\alpha$ controls how quickly knowledge infusion takes place. 
 \paragraph{$P(\Psi_{k}(i)) = Laplace(x,b)$: }Specifying $\alpha$ is a tricky thing to do for a human and we would like them to able to just simply specify preference over arm choice given a context instead, if they are an expert. To model an expert \begin{itemize}
 \item First we set $\Psi_{k}(i)_{knowledge} = {\rm LUB}\{\alpha\}$, where ${\rm LUB}\{\alpha\}$ stands for the least upper bound from among a set of $\alpha \in \mathbf{\{\alpha\}}$. The interpretation is that $\alpha$ has to be at least that high to qualify as expert knowledge. We set ${\rm LUB}\{\alpha\} = K\cdot\max{\pi_{k}(i)\nabla_{\Psi_{k}(i)}\log(\pi_{k}(i))r_{k}(i)} = K\cdot 1 \cdot K = K^{2}$ as the maximum value of $\pi_{k}(i) = 1$ and the maximum value of $\nabla_{\Psi_{k}(i)}\log(\pi_{k}(i))\cdot r_{k}(i)$ is $1\cdot K$ as the maximum value of $\sum_{k=1}^{K}r_{k}(i) = K$. Thus we set $\Psi_{k}(i)_{knowledge} = {\rm LUB}\{\alpha\} = K^{2}$. The interpretation is the human has to be at least as sure as the correction required to the error in arm choice i.e. the max gradient to qualify as an expert. Therefore to prefer arm $i$, $\alpha = K^{2}$ and if arm $i$ is not preferred, $\alpha = -K^{2}$.
 \item Next, we replace the $Normal(\mu,\Sigma)$ distribution with the $Laplace(x = |\Psi_{k}(i)_{knowledge}-\Psi_{k}|,b=1)$ distribution. Thus, we obtain that $\nabla_{\Psi_{k}(i)}\log(P(\Psi_{k}(i)) = {\rm sign}(\Psi_{k}(i)_{knowledge} - \sigma(\Psi_{k}(i))) = \pm 1$. If the expert prefers the arm $i$, $\delta_{k}(i) = \delta_{k}(i) + 1$ and if the expert does not prefer the arm $i$, $\delta_{k}(i) = \delta_{k}(i) - 1$. This is very intuitive as it means that the $\Psi_{k}(i)$, representing chance of arm $i$ being pulled is simply increased or decreased by an additive factor depending on expert's preference, thus preventing the need to carefully specify $\alpha$.
 \item With this insight, it suffices for the human expert to specify knowledge as a tuple \[\mathbf{knowledge: (c_{k}(i), prefer(i) = \{0,1\})},\] which simply means that at step $k$, given the context $c_{k}(i)$, arm $i$ is either preferred ($prefer(i) = 1)$) or not preferred ($prefer(i) = 0$). This is much more natural and easy for the expert human to specify. Note that if the human had a reason to specify $\alpha$ quantifying how quickly they want the knowledge infusion to take place depending on how sure they are (expert level), we can use the $Normal$ or $Laplace$ distribution form to specify without the use of ${\rm LUB}\{\alpha\}$. Algorithm \ref{alg:BKIPG} shows the pseudocode for KIPG with expert knowledge infusion.
 Also, we add $1$ to $r_{k}(i)$ so that the gradient doesn't vanish when $r(i) = 0$.
 \end{itemize}
 
  \subsubsection{Example of knowledge in the IMDB dataset using the Laplacian Formulation} At a step $k$, we can define knowledge over
  the actors set $\mathbf{A} =x \{actor1,actor2,actor3,..\}$ with respect to a directors set $\mathbf{D}=\{director1,director2,..\}$ and a movies set $\mathbf{M}=\{movie1,movie2,..\}$ as, \[\mathrm{(directed(\mathbf{D},\mathbf{M}) \land actedIn(\mathbf{A},\mathbf{M}), prefer(workedUnder(\mathbf{A},\mathbf{D})) = 1)}.\] This means that \emph{The set of actors $\mathbf{A}$, worked under the set of directors $\mathbf{D}$, in the movies in the set $\mathbf{M}$}. In this example, $(directed(\mathbf{D},\mathbf{M}) \land actedIn(\mathbf{A},\mathbf{M})$ is the context $c(i)$, $i$ is the arm label workedUnder.  
 \begin{algorithm}[t]
\caption{Knowledge Infused Policy Gradients - KIPG}\label{alg:BKIPG}
\begin{algorithmic}[1]
\State Initialize $\Psi_{0}(i) = 0~\forall$ arms $i$
\For{$k \gets 1$ to $K$}
\State set $\pi_{k}(i) = \sigma(\Psi_{k-1}(i))$
\State Draw arm $i^{*} = \argmax_{i}i\sim\pi_{k}(i)$
\Comment observe reward $r_{k}(i^{*})$ and context $c_{k}(i^{*})$
\State Compute $\nabla_{\Psi_{k}(i^{*})}\log(\pi_{k}(i^{*}))$ as
\Comment{$\pm$ Depending on preference}
\begin{equation*}
    \begin{aligned}
    (I_{k}(i^{*})-\pi_{k}(i^{*}) \pm 1)
    \end{aligned}
\end{equation*}
\State Compute gradient as $\pi_{k}(i^{*})\nabla_{\Psi_{k}(i^{*})}\log(\pi_{k}(i^{*}))(r_{k}(i^{*}) + 1)$
\Comment add $1$ smoothing
\State Fit $\delta_{k}(i^{*})$ to gradient using TILDE tree
\State Set $\Psi_{k}(i^{*}) = \Psi_{k-1}(i^{*}) + \eta\delta_{k}(i^{*})$
\EndFor
\State return $\pi_{K}(i)$
\end{algorithmic}
\end{algorithm} 
 \subsubsection{Connection with Previous Work on Relational Preferences}
 Odom et al. \cite{odom2015knowledge} have previously specified relational preference knowledge in supervised learning and imitation learning settings. Using their approach, at step $k$, the knowledge would be incorporated by an additive term to the gradient term $(I_{k}(i)-\sigma(\Psi_{k}(i)))$. This term is $n_{k}(i)_{t} - n_{k}(i)_{f}$, where $n_{k}(i)_{t}$ is the number of knowledge sources that prefer arm $i$ and $n_{k}(i)_{f}$ is the number of knowledge sources that do not prefer arm $i$, at step $k$.  We prove in Theorem \ref{theorem: th1} that the approach of Odom et al. \cite{odom2015knowledge} is a specific instance of KIPG with multiple knowledge sources. For our experiments, we specify only a single source of knowledge at all steps $k$.

 \begin{theorem}\label{theorem: th1}
 At step $k$, For $S$ multiple knowledge sources, that either prefer or don't prefer arm $i$, $k1,k2,..kS$, assuming independence, let $P(\Psi_{k}(i)) = \prod_{s=1}^{S}Laplace(|\Psi_{k}(i)-\Psi_{k}(i)_{ks}|,b=1)$. Here $\Psi_{k}(i)_{ks} = \Psi_{k}(i)_{knowledge}~\forall s \in S$. Then we have $\nabla_{\mathbf{\Psi_{k}}}\log(\pi_{k}(i)) = n_{k}(i)_{t} - n_{k}(i)_{f}$.
 \end{theorem}
 \begin{proof}
 We know that with assuming a $Laplace(x,b)$ distribution and setting $\Psi_{k}(i)_{ks} = \Psi_{k}(i)_{knowledge} = {\rm LUB}\{\alpha\}~\forall s \in S$, we get $\nabla_{\Psi_{k}(i)}\log(P(\Psi_{k}(i))) = \sum_{s=1}^{S}{\rm sign}({\rm LUB}\{\alpha\}-\sigma(\Psi_{k}(i)))$. We know also that ${\rm sign}({\rm LUB}\{\alpha\}-\sigma(\Psi_{k}(i))) = \pm 1$ depending on if the expert prefers the arm $i$ or not. Thus we get,$\sum_{s=1}^{S}{\rm sign}({\rm LUB}\{\alpha\}-\sigma(\Psi_{k}(i))) =  n_{k}(i)_{t} - n_{k}(i)_{f}$. 
 \end{proof}

\subsubsection{Connection with Thompson Sampling}
We now formalize the connection between Thompson Sampling with posterior reshaping and KIPG. For arm $i\in[N]$, at every step of arm pulling $k \in [K]$, a reward $r_{k}(i)$ and a context $c_{k}(i)$ is emitted. In Thompson Sampling, the posterior $P(\Theta_{k}(i)|r_{k}(i),c_{k}(i))$ for parameter $\Theta_{k}(i)$ representing $P(r_{k}(i)|c_{k}(i))$ is updated at each step $k$ as 
\[\frac{P(r_{k}(i)|\Theta_{k}(i),c_{k}(i))\Pr(\Theta_{k}(i)|c_{k}(i))}{\int_{\Theta_{k}(i)}P(r_{k}(i)|\Theta_{k}(i),c_{k}(i))\Pr(\Theta_{k}(i)|c_{k}(i))}.\]
 Finally, the optimal arm choice corresponds to the arm that has the max among the sampled $\Theta_{k}(i) \sim P(\Theta_{k}(i)|r_{k}(i),c_{k}(i))$ for each arm $i$.
 \sloppy The posterior $P(\Theta_{k}(i)|r_{k}(i),c_{k}(i))$, can be reshaped for example by using $P(\Theta_{k}(i) = \mathbf{F}(\Theta_{k}(i)|r_{k}(i),c_{k}(i))$. The reshaping changes the sufficient statistics such as mean, variance, etc. This $\mathbf{F}$ can be informed by some knowledge of the domain. We encounter a couple of issues with Posterior Reshaping for knowledge infusion. First, that the choice of $\mathbf{F}$ is difficult to determine in a principled manner. Second, the choice of $\mathbf{F}$ must be determined such that it is amenable to sampling for exploration. Sampling itself is very inefficient for problems of appreciable size. Thus, we observe that we can instead directly optimize for the optimal arm choice through policy gradient methods. Using a Bayesian formulation for optimization of policy in functional space, we can see that the reshaped posterior after $K$ iterations of arm pulling (where $K$ is sufficiently high), corresponds to learning an optimal function $\Psi(i)$ since $\Psi(i)$ is high if $\mathbf{F}(\Theta_{k}(i)|r_{k}(i),c_{k}(i))$, representing $P(r(i)=1|c(i))$, is high.

\subsubsection{Connection with Exp3}
Exp3 maximizes the total expected reward over $K$ arm pulls $ f = \sum_{k=1}^{K}\pi_{k}(i)r_{k}(i)$. Using the proximal definition of gradient descent and deriving the mirror descent objective after each arm pull, we have
 \[\pi_{k}(i) = \argmax_{\pi(i)}((\gamma\cdot\pi(i)\cdot\nabla_{\pi_{k-1}(i)}(f)) + \mathcal{D}(\pi_{k-1}(i),\pi(i))).\], where $\gamma$ is the learning rate. Choosing $\mathcal{D}(\pi(i),\pi_{k-1}(i)) = \Phi(\pi_{k-1}(i)) - (\Phi(\pi(i)) + \nabla\Phi(\pi_{k-1}(i))(\pi_{k-1}(i)-\pi(i)))$, where $\Phi$ is a convex function, we get
 \[\nabla\Phi(\pi_{k}(i)) = \nabla\Phi(\pi_{k}(i)) + \gamma\cdot\nabla_{\pi_{k-1}(i)}(f).\]
 Since $\pi$ is a probability we need to choose a convex $\Phi$ such that it works with probability measures. So we will choose $\Phi(\pi) = \sum_{i}\pi(i) \log \pi(i)$ to be negative entropy and we have
 \[\log(\pi_{k}(i)) = \log(\pi_{k-1}(i)) + \gamma\cdot\nabla_{\pi_{k-1}(i)}(f).\] 
 Setting $\pi_{k-1}(i) = \sigma(\Psi_{k}(i))$, we get,
 \[\log(\pi_{k}(i)) \propto \log(\sigma(\Psi_{k}(i))) + \log(e^{\gamma\cdot\nabla_{\pi_{k-1}(i)}(f)}),\] where $\log P(\Psi_{k}(i)) = \log(e^{\gamma\cdot\nabla_{\pi_{k-1}(i)}(f)})$. Thus we see that \emph{Exp3 can be seen as a case of applying a specific prior probability in KIPG}.

 \section{Regret Bound for KIPG}
 We now derive a bound for the total regret after $K$ steps of KIPG to understand the convergence of KIPG towards the optimal arm choice. Since KIPG is fundamentally a gradient ascent approach, we can use analysis similar to the regret analysis for online gradient ascent to derive the regret bound \cite{hazan2008adaptive}. Using $a^{2}-(a-b)^{2}=2ab-b^{2}$ and letting $a=(\Psi_{k}(i)-\Psi^{*}(i))$ and $b = \nabla_{\Psi(i)_{k}}\sum_{k=1}^{K}\pi_{k}(i)r_{k}(i)$, We know that for a sequence over $K$ gradient ascent iterations, $\{\Psi_{k}(i)|k\in[K]\}$, we have
 \begin{align*}
 (\Psi_{k}(i)-\Psi^{*}(i))^{2}
 \leq (\Psi_{k-1}(i)-\Psi^{*}(i))^{2} - 2\gamma(\pi_{k}(i)r_{k}(i) -  \pi^{*}(i)r(i^{*})) + \gamma^{2}\mathcal{L}
 \end{align*}
 where $\mathcal{L} \geq \nabla_{\Psi_{k}(i)}\sum_{k=1}^{K}\pi_{k}(i)r_{k}(i)$ is an upper bound on the gradient (Lipschitz constant) 
 and $\gamma$ is the learning rate. 
 Using a telescoping sum over $K$ iterations we have
 \begin{align*}
     (\Psi_{K}(i)-\Psi^{*}(i))^{2} \leq (\Psi_{0}(i)-\Psi^{*}(i))^{2}  - 2\sum_{k=1}^{K}(\gamma(\pi_{k}(i)r_{k}(i) - \pi^{*}(i)r(i^{*}))) + \sum_{k=1}^{K}\gamma^{2}\mathcal{L}
 \end{align*}
and therefore
\[\sum_{k=1}^{K}(\gamma(\pi_{k}(i)r_{k}(i) - \pi^{*}(i)r(i^{*}))) \leq \frac{\max_{\Psi_{k}(i)}(\Psi_{k}(i)-\Psi^{*}(i))^{2} + \mathcal{L}^{2}\sum_{k=1}^{K}\gamma^{2}}{2\sum_{k=1}^{K}\gamma}.\] Solving for $\gamma$ by setting $\nabla_{\gamma}(R.H.S) = 0$, we finally have our total regret bound over $K$ steps as:
\[\sum_{k=1}^{K}(\gamma(\pi_{k}(i)r_{k}(i) - \pi^{*}(i)r(i^{*}))) \leq \frac{\max_{\Psi_{k}(i)}(\Psi_{k}(i)-\Psi^{*}(i))^{2}\mathcal{L}}{\sqrt{K}}.\] 
This regret bound has a very intuitive form. It shows that the regret is bounded by how far off the learned $\Psi(i)$ from the true $\Psi^{*}(i)$ for each arm $i$. Thus we expect that in the experiments, with quality knowledge infusion this gap is drastically reduced over $K$ steps to result in a low total regret.
\section{KIPG-Upper Confidence Bound}
So far we have developed KIPG for the Bandit Setting and derived a regret bound. Since KIPG estimates $\pi(i)$ after each arm pull, we can sample from $\pi(i)$ and choose the max like in Thompson Sampling. However, since the arm to pull is being learned online, the uncertainty in the arm choice even with knowledge needs to be modeled. The human providing knowledge needs to observe a few user-arm pulls to gradually improve their confidence in the knowledge provided. As the data is not available offline to study by the human, it is unlikely that the knowledge provided is perfect initially. Thus, we now derive a confidence bound to quantify the uncertainty in the arm choice.
At step $k$, let the arm choice be denoted by $i^{*}$.
First we notice that $Z = |\pi_{k}(i^{*})-\pi^{*}(i^{*})|$, is binomial distributed at step $k$. Also, $\pi_{k}(i^{*})$ is binomial distributed. However, for both we will use a Gaussian approximation and note that for this Gaussian, $\mu(Z) = 0$ and $\sigma(Z) \leq \mathbb{E}[(\pi_{k}(i^{*}) - \pi^{*}(i^{*}))^{2}]$, thus making this a \emph{sub-Gaussian}\cite{peizer1968normal,buldygin1980sub}. Using Markov's inequality we have \cite{cohen2015markov}:
\[P(Z > \epsilon) \leq e^{-k\epsilon}\mathbb{E}[kZ]\implies P(e^{kZ} > e^{k\epsilon}) \leq \mathbb{E}[e^{kZ}]\cdot e^{-k\epsilon}\] where $e^{kZ}$ is the moment-generating-function for $Z$. We know that $e^{kZ}$ is convex and thus $e^{kZ} \leq \gamma(e^{kb}) + (1-\gamma)e^{ka}$ for $Z \in [a,b]$ and $\gamma \in [0,1]$. Thus we obtain
$Z \leq \gamma b + (1-\gamma)a$, which gives us
$\gamma \geq \frac{Z-a}{b-a}$, therefore we know
\[e^{kZ} \leq \frac{-ae^{kb}+be^{ka}}{b-a} + \frac{Z(e^{kb}-e^{ka})}{b-a}.\]
Taking expectation on both sides we get
$\mathbb{E}[e^{kZ}] \leq \frac{-ae^{kb}+be^{ka}}{b-a}$.
Let $e^{g(k)} = \frac{-ae^{kb}+be^{ka}}{b-a}$, we get $g(k) = ka + \log(b-ae^{k(b-a)}) - \log(b-a)$. 
Using Taylor series expansion for $g(k)$ upto the second order term as $g(0) + \nabla(g(k))k + \frac{\nabla^{2}(g(k))k^{2}}{2}$, 
we get
\[\nabla^{2}(g(k)) = \frac{ab(b-a)^{2}(-e^{k(b-a)})}{(ae^{k(b-a)}-b)^{2}}.\]
We note that
$
    ae^{t(b-a)} \geq a
    \implies ae^{t(b-a)}-b \geq a-b
    \implies (ae^{t(b-a)}-b)^{-2} \leq (b-a)^{-2}
$.
We know $-e^{k(b-a)} \leq -1$, therefore we obtain
\begin{align*}
    &\nabla^{2}(g(k)) \leq \frac{-ab(b-a)^{2}}{(b-a)^{2}} = -ab \leq \frac{(b-a)^{2}}{4}
    \implies g(k) \leq \frac{(b-a)^{2}}{4}\frac{k^{2}}{2}.
\end{align*}
We know that
$
    \mathbb{E}[e^{kZ}] \leq e^{g(k)}
    \implies \mathbb{E}[e^{kZ}] \leq e^{\frac{k^{2}(b-a)^{2}}{8}}
$.
Once again from the Markov inequality, we have
\[P(Z > \epsilon) \leq e^{-k\epsilon}\mathbb{E}[kZ]\implies P(|\pi_{k}(i^{*}) - \pi^{*}(i^{*})| > \epsilon) \leq e^{-k\epsilon + \frac{k^{2}(b-a)^{2}}{8}}.\] 
Using $k = \frac{4\epsilon}{(b-a)^{2}}$, by solving for the minimum of
$e^{-k\epsilon + \frac{k^{2}(b-a)^{2}}{8}}$ we get
$P(|\pi_{k}(i^{*}) - \pi^{*}(i^{*})| > \epsilon) \leq e^{\frac{-2\epsilon^{2}}{(b-a)^{2}}}$. 
As $0 \leq (b-a) \leq 1$, we have
$P(|\pi_{k}(i^{*}) - \pi^{*}(i^{*})| > \epsilon) \leq e^{-2\epsilon^{2}}$ and, after $K$ time steps,
\[P(|\pi_{K}(i^{*}) - \pi^{*}(i^{*})| > \epsilon) \leq e^{-2K\epsilon^{2}}.\] Solving for $\epsilon$ we get, $\epsilon \leq \frac{-\log(P(|\pi_{K}(i^{*}) - \pi^{*}(i^{*})| > \epsilon))}{2K}$.
Thus, we draw the next optimal arm choice $i$ at $k+1$ as follows:
\[\argmax_{i}\bigg\{i \sim \pi_{k+1}(i) = \sigma(\Psi_{k}(i) + \frac{-\log(P(Z > \epsilon))}{2k})\bigg\},\] where $Z=|\pi_{k}(i^{*}) - \pi^{*}(i^{*})|$.
This confidence bound also has an intuitive form as it is reasonable that the expectation $\mathbb{E}(I(|\pi_{k}(i^{*}) - \pi^{*}(i^{*})| > \epsilon))$ gets closer to the truth (less probable) as more arms are pulled, where $I$ is the indicator function.
Since we never actually know $\pi^{*}(i^{*})$, we set to the current best arm choice. We expect that knowledge infusion will allow the error between the current best arm choice and $\pi^{*}(i^{*})$ to get smaller. As $P$ is usually initially set high and decayed as $k$ increases causing $\log(P)$ to increase, we achieve this effect by simply using $-\log(|\pi_{k}(i^{*}) - \pi^{*}(i^{*})|)$. Algorithm \ref{alg:KIPGUCB} shows how a simple modification to the pseudocode in Algorithm \ref{alg:BKIPG} can incorporate the bound derived. 

 \begin{algorithm}
\caption{KIPG Upper Confidence Bound - KIPGUCB}\label{alg:KIPGUCB}
\begin{algorithmic}[1]
\State Initialize $\Psi_{0}(i) = 0~\forall$ arms $i$
\For{$k \gets 1$ to $K$}
\State set $\pi_{k}(i) = \sigma(\Psi_{k-1}(i))$
\State Draw arm $i^{*} = \argmax_{i}i\sim\pi_{k}(i)$
\Comment observe reward $r_{k}(i^{*})$ and context $c_{k}(i^{*})$
\State Set $\pi^{*}(i^{*}) = I(\pi_{k}(i^{*}) = i^{*})$
\State Compute $\nabla_{\Psi_{k}(i^{*})}\log(\pi_{k}(i^{*}))$ as
\Comment{$\pm$ Depending on preference}
\begin{equation*}
    \begin{aligned}
    \bigg(I_{k}(i^{*})-\pi_{k}(i^{*}) \pm 1 - \frac{\log(|\pi_{k}(i^{*}) - \pi^{*}(i^{*})|)}{2k}\bigg)
    \end{aligned}
\end{equation*}
\State Compute gradient as $\pi_{k}(i^{*})\nabla_{\Psi_{k}(i^{*})}\log(\pi_{k}(i^{*}))(r_{k}(i^{*}) + 1)$
\Comment add $1$ smoothing
\State Fit $\delta_{k}(i^{*})$ to gradient using TILDE tree
\State Set $\Psi_{k}(i^{*}) = \Psi_{k-1}(i^{*}) + \eta\delta_{k}(i^{*})$
\EndFor
\State return $\pi_{K}(i)$
\end{algorithmic}
\end{algorithm}
\section{Experiments}{\label{sec:exp}}
The knowledge used in our experiments comes from domain experts, an example of which is seen in Section \ref{sec:four}. We aim to answer the following questions:
\begin{enumerate}
    \item How effective is the knowledge for bandit arm selection?
    \item How effective is the UCB exploration strategy for bandit arm selection?
\end{enumerate}
\subsection{Simulated Domains}
\paragraph{Simulation model: }We perform experiments on a simulated music recommendation dataset. The dataset simulates songs, artists, users, and albums where there are the following user behaviors:
\begin{itemize}
    \item Behavior A: The users are fans of one of the artists in the dataset.
    \item Behavior B: The users follow the most popular song.
    \item Behavior C: They follow the most popular artist.
\end{itemize}
We will denote the set of behaviors by $\mathbf{Behaviors}$.
Figure \ref{fig:schema}(b) shows an illustration for the Schema for the simulation model depicting that $M$ users can listen to $N$ songs and $N$ songs can be sung by $N$ artists, etc. Artists and Songs have attributes ``Popular'' denoting if a particular artist or a song is popular among users. 
\begin{figure*}[t]
    \centering
    \includegraphics[width=\textwidth]{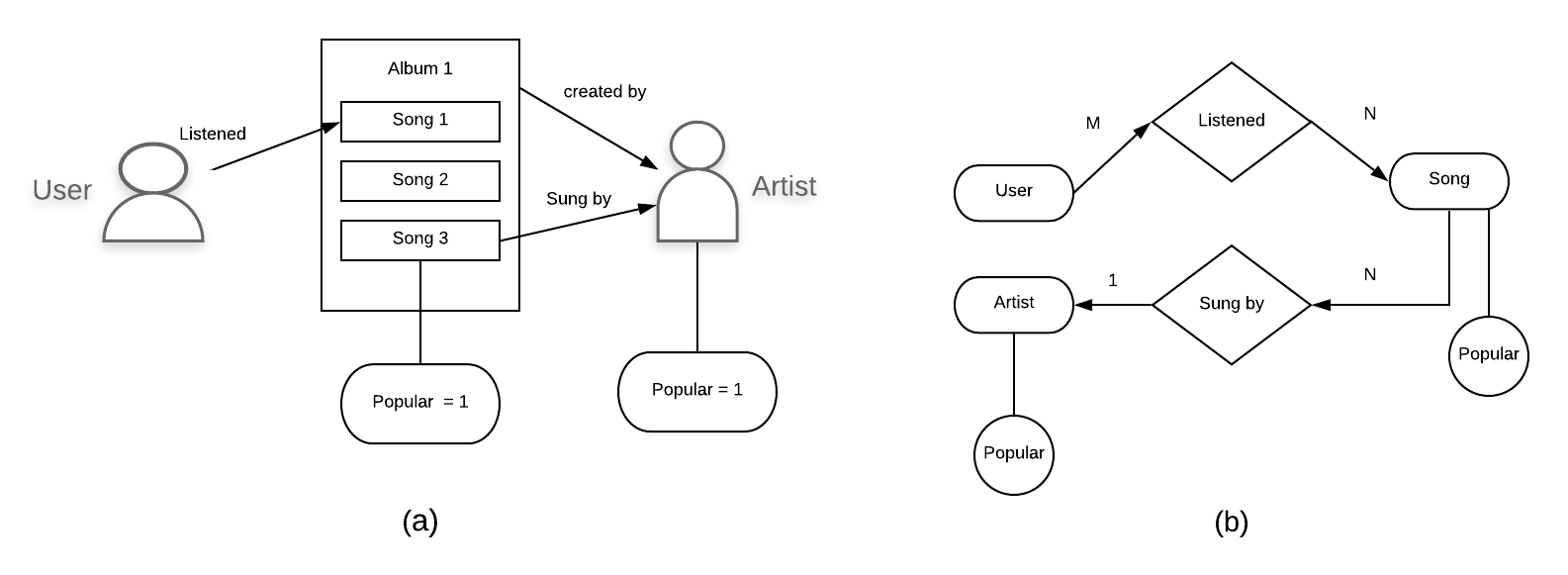}
    \caption{Illustration of the Entity-Relationship Schema diagram for the Music Recommendation system being simulated (b) and a particular instantiation (a). Users listen to songs by artists. $M$ Users can listen to $N$ Songs, $N$ Songs can be written by $1$ Artist and Artists and Songs can be popular among Users.}
    \label{fig:schema}
\end{figure*}
\paragraph{Context Induction: } Once the simulation model is used to generate different users based on a predefined behavior $\in \mathbf{Behaviors}$. We need now to generate different possible user contexts from this dataset. Since the whole dataset is not available to us offline, we construct a dataset by $50$ random arm choices to induce contexts. The contexts will be represented using predicate logic clauses: antecedent ($\land$ preconditions representing possible user context) $\implies$ consequent (user song choice). For this, an inductive bias needs to be provided to induce sensible clauses. Such an inductive bias is included as background knowledge to the induction program. We use the method in Hayes et al. \cite{hayes2017user} to automatically construct the inductive bias from the schema in Figure \ref{fig:schema}(b). The clauses induced are kept if they satisfy minimum information criteria i.e. if they discriminate at least one user from another in their song choice, in the dataset. The clauses induced using the provided inductive bias and are as follows:
\begin{itemize}
    \item sungBy(B,C) $\land$ $\lnot$ popular(C)$\implies$listens(A,B). This context says \emph{User A listens to song B if song B is sungBy artist C. Also, C is not a popular artist}, which describes \textbf{behavior A}.
    \item sungBy(B,C) $\land$ popular(C)$\implies$listens(A,B). This context says \emph{User A listens to song B if song B is sungBy a popular artist C}, which describes \textbf{behavior C}. 
    \item listened(C,B)$\implies$listens(A,B). This context says \emph{user A listens to song B if user C listened to B}, which describes \textbf{behavior B}.
\end{itemize}
We use satisfiability of these clause antecedents as features for TILDE regression tree stumps. Figure \ref{fig:stump} shows an example, where sigmoid of the regression values represents arm choice probability $\pi(i)$.
\begin{figure}[!h]
    \centering
    \includegraphics{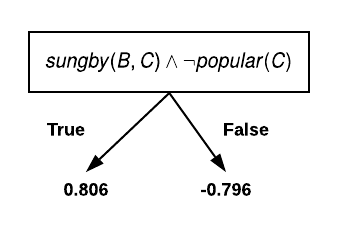}
    \caption{
    Example of a TILDE regression tree stump for song choice. The tree depicts that if \emph{if song B is sungBy artist C and also, C is not a popular artist}, \emph{User A listens to B} with probability $\sigma(0.806)$. Else, \emph{User A listens to B} with probability $\sigma(-0.796)$.}
    \label{fig:stump}
\end{figure}
\subsubsection{Results}
We compare the RB2 algorithm with KIPG and KIPGUCB. For each type of user, at time step $k$, a recommendation is provided depending on the algorithm used. The regret drawn from comparison to the ground truth (GT) recommendation is recorded. The regret equation for an algorithm $\mathcal{A}$ is:
\[R_{\mathcal{A}} = \sum_{k=1}^{K}(r^{GT} - \pi_{k}(i^{*})_{{\mathcal{A}}}r_{k}(i^{*})),\] 
where $i^{*}$ is the optimal arm drawn from $\argmax$ over $\pi(i)$ samples at step $k$ (See Algorithm \ref{alg:BKIPG},\ref{alg:KIPGUCB} - line 4). $r^{GT}$ is the reward if the ground truth optimal arm is drawn at $k$.
\paragraph{Perfect Knowledge: } The human providing knowledge may have some previous knowledge about a user in the system. In this case, it is expected that the knowledge is pretty good from the start. In this setting, we expect the regret is ordered as $R_{KIPG} < R_{KIPGUCB} < R_{RB2}$ for most $k = 1$ to $K$.We expected this trend since RB2 uses no knowledge and KIPGUCB moves slower towards knowledge initially. Given that the knowledge is perfect, we expect KIPG to perform the best. We set $K = 500$. Figure \ref{fig:perf} shows that the experiments corroborate this.

\begin{figure*}[!ht]
     \centering
     \begin{subfigure}[b]{0.3\textwidth}
         \centering
         \includegraphics[width=\textwidth]{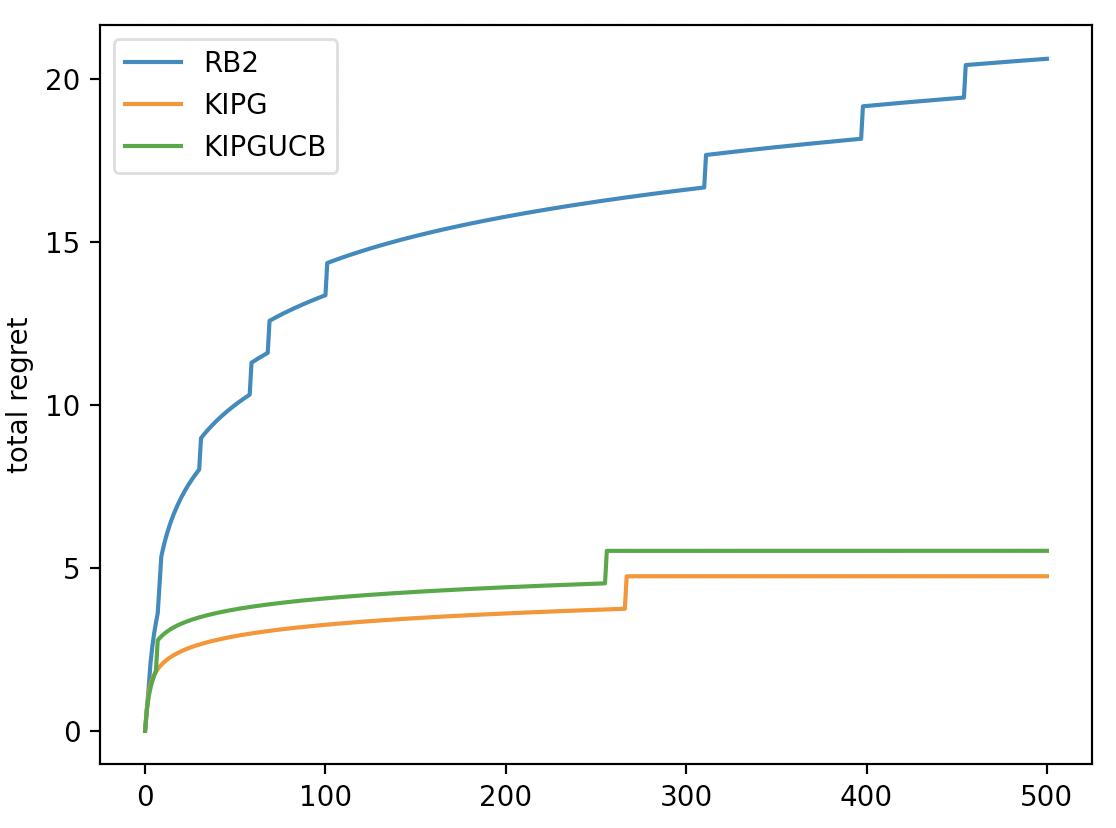}
         \caption{Behavior A}
         \label{fig:behap}
     \end{subfigure}%
     ~
     \begin{subfigure}[b]{0.3\textwidth}
         \centering
         \includegraphics[width=\textwidth]{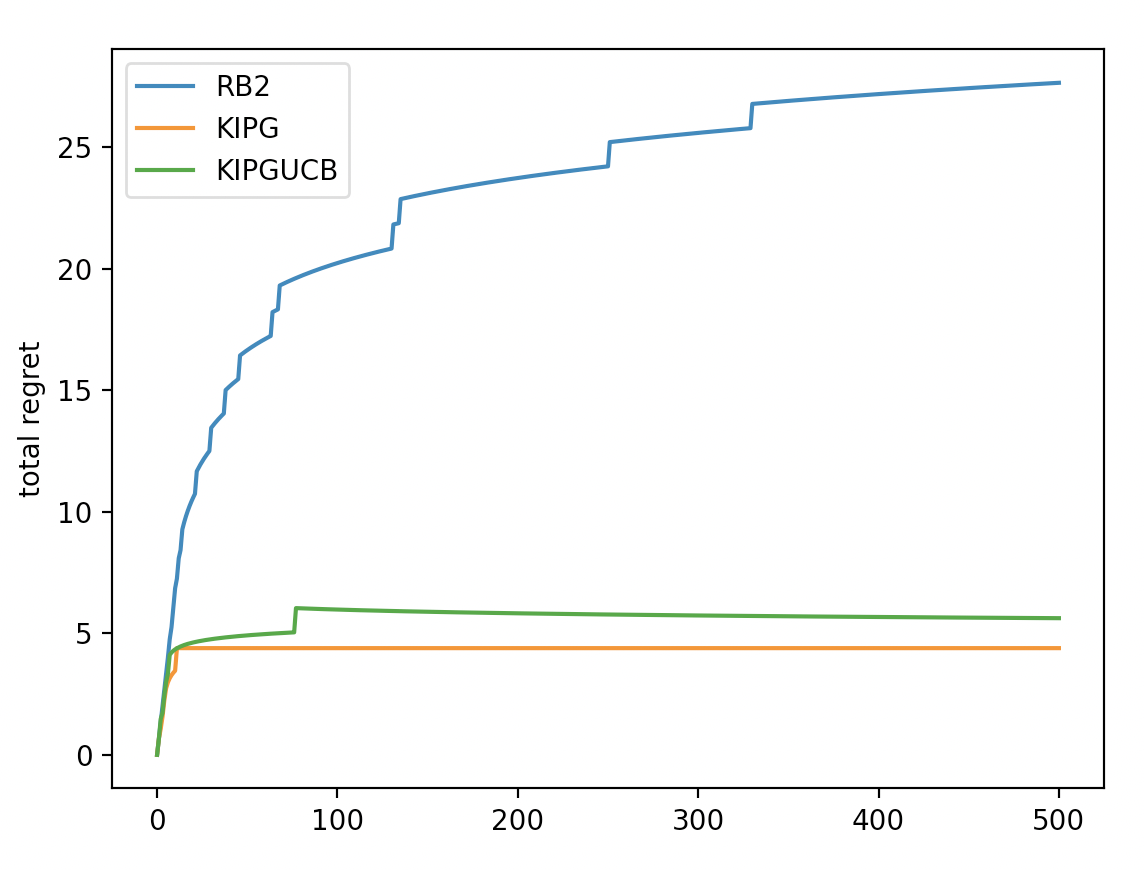}
         \caption{Behavior B}
         \label{fig:behbp}
     \end{subfigure}%
     ~
     \begin{subfigure}[b]{0.3\textwidth}
         \centering
         \includegraphics[width=\textwidth]{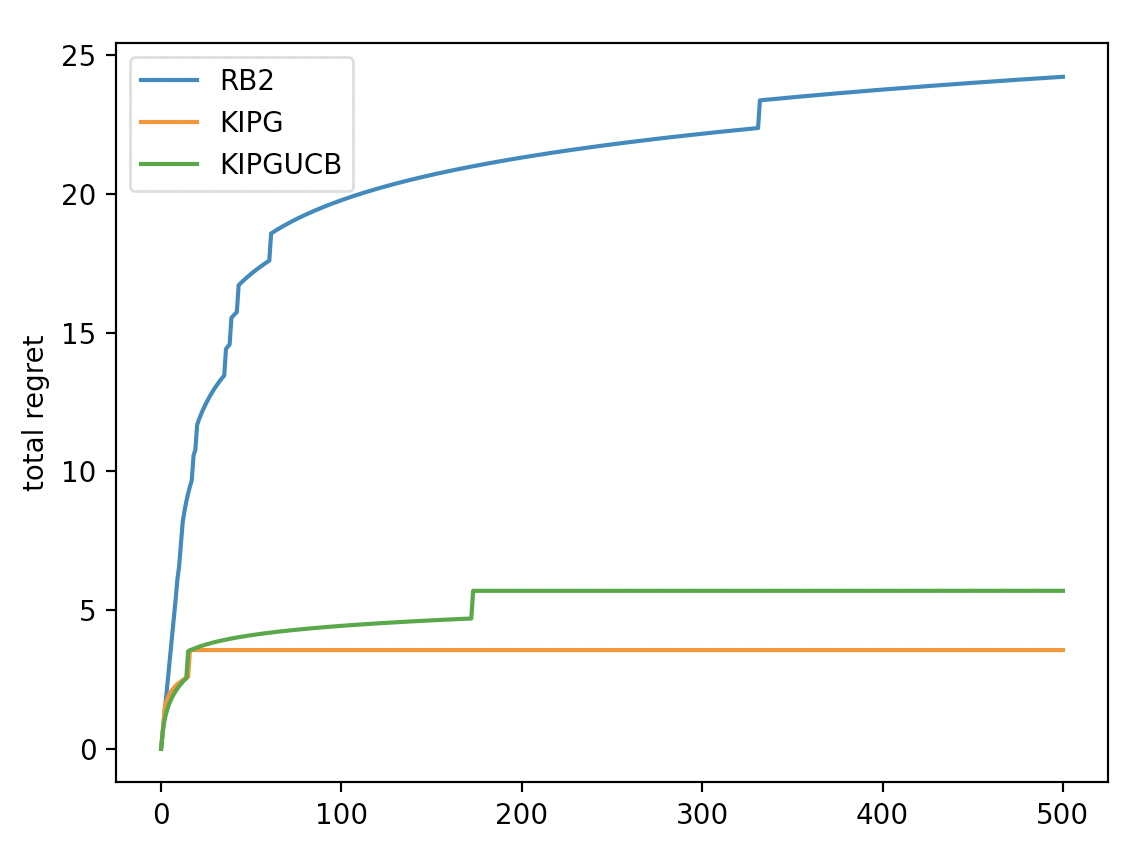}
         \caption{Behavior C}
         \label{fig:behcp}
     \end{subfigure}
    \caption{Shows comparison of $R_{RB2}$, $R_{KIPG}$, $R_{KIPGUCB}$ for the perfect knowledge setting for all their behaviors. As expected we see that $R_{KIPG} < R_{KIPGUCB} < R_{RB2}$ for most $k = 1~to~K$.}
        \label{fig:perf}
\end{figure*}

\begin{figure*}[!ht]
     \centering
     \begin{subfigure}{0.3\textwidth}
         \centering
         \includegraphics[width=\textwidth]{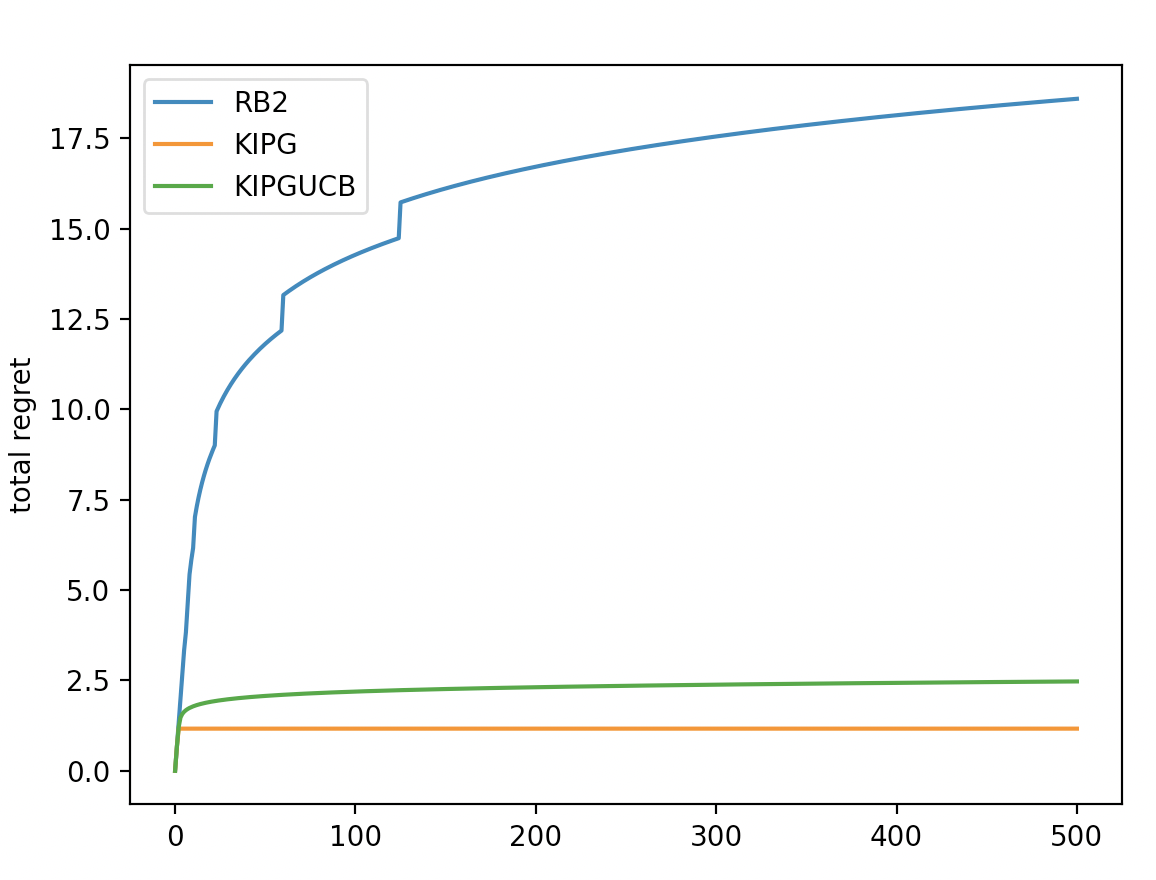}
         \caption{Behavior A}
         \label{fig:behanp}
     \end{subfigure}%
     ~
     \begin{subfigure}{0.3\textwidth}
         \centering
         \includegraphics[width=\textwidth]{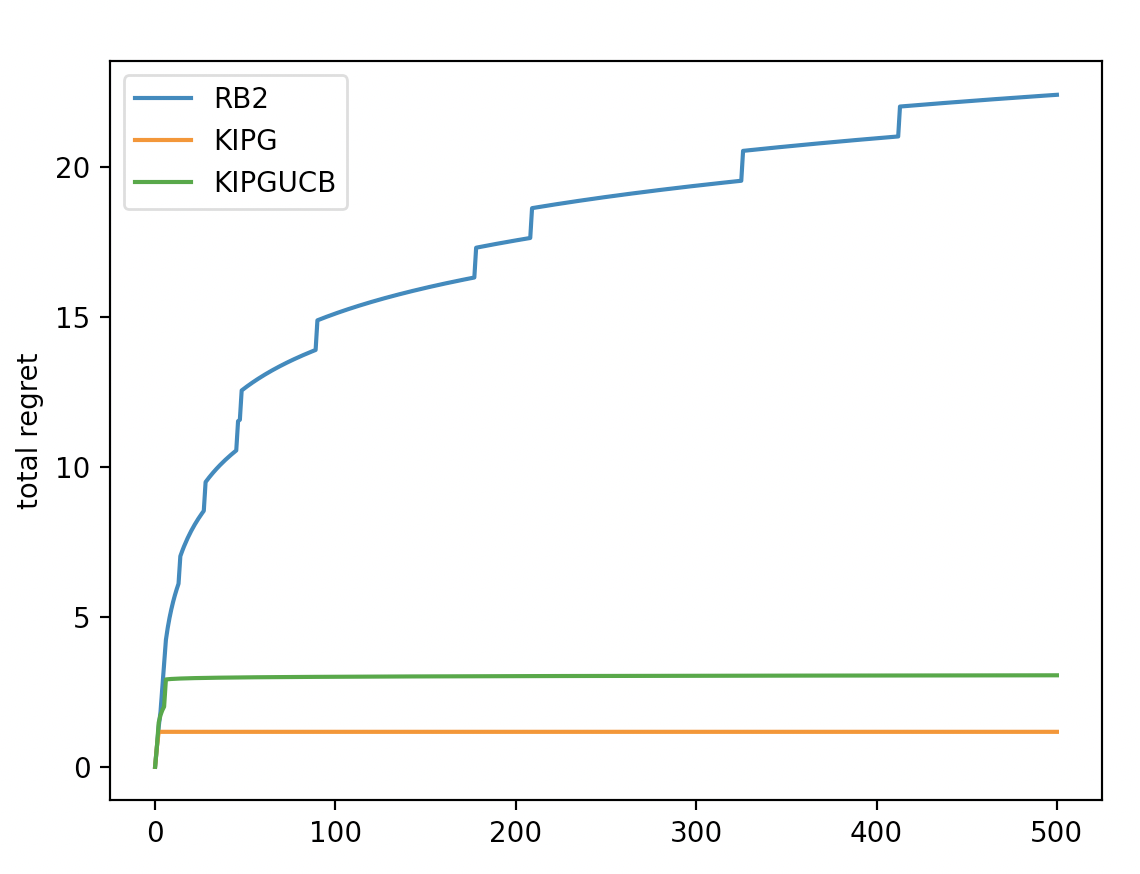}
         \caption{Behavior B}
         \label{fig:behbnp}
     \end{subfigure}%
     ~
     \begin{subfigure}{0.3\textwidth}
         \centering
         \includegraphics[width=\textwidth]{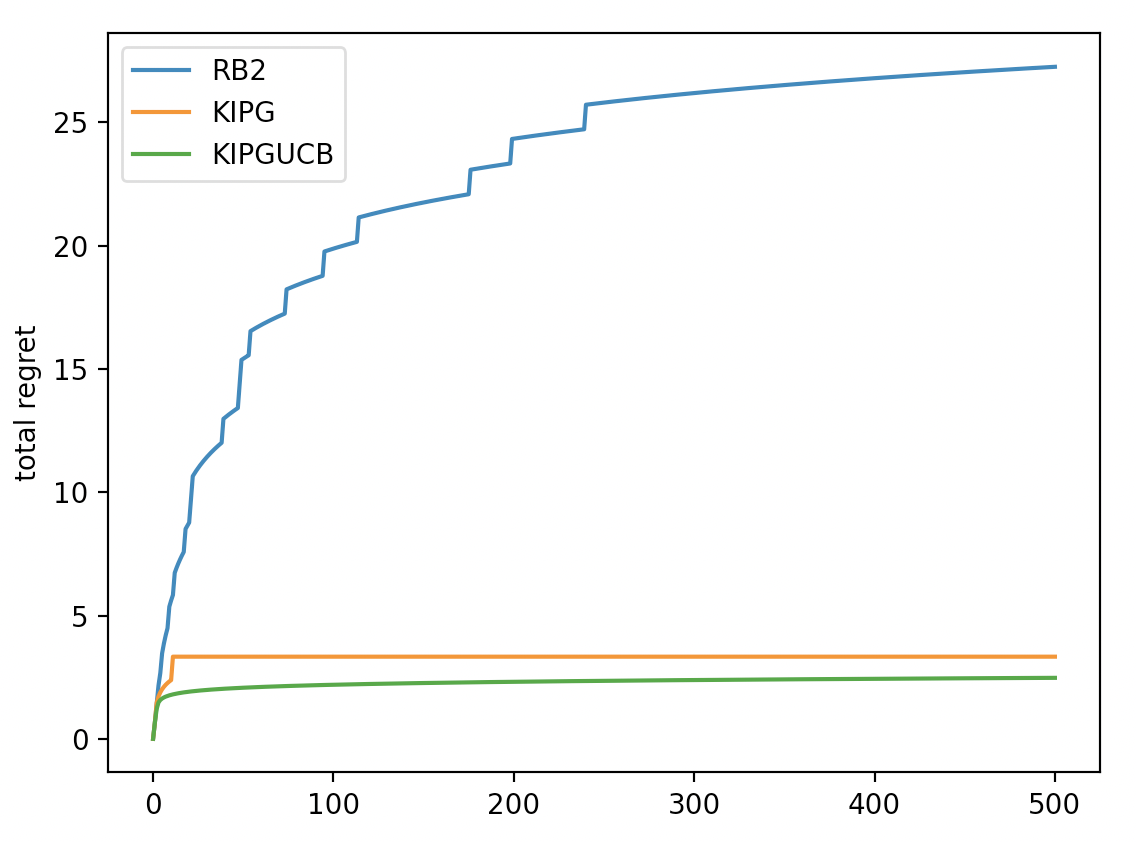}
         \caption{Behavior C}
         \label{fig:behcnp}
     \end{subfigure}
    \caption{Shows comparison of $R_{RB2}$, $R_{KIPG}$, $R_{KIPGUCB}$ for the nearly perfect knowledge setting for all three behaviors. As expected we see that $R_{KIPG} < R_{KIPGUCB} < R_{RB2}$ for most $k = 1~to~K$.}
        \label{fig:nperf}
\end{figure*}

\begin{figure*}[!ht]
     \centering
     \begin{subfigure}{0.3\textwidth}
         \centering
         \includegraphics[width=\textwidth]{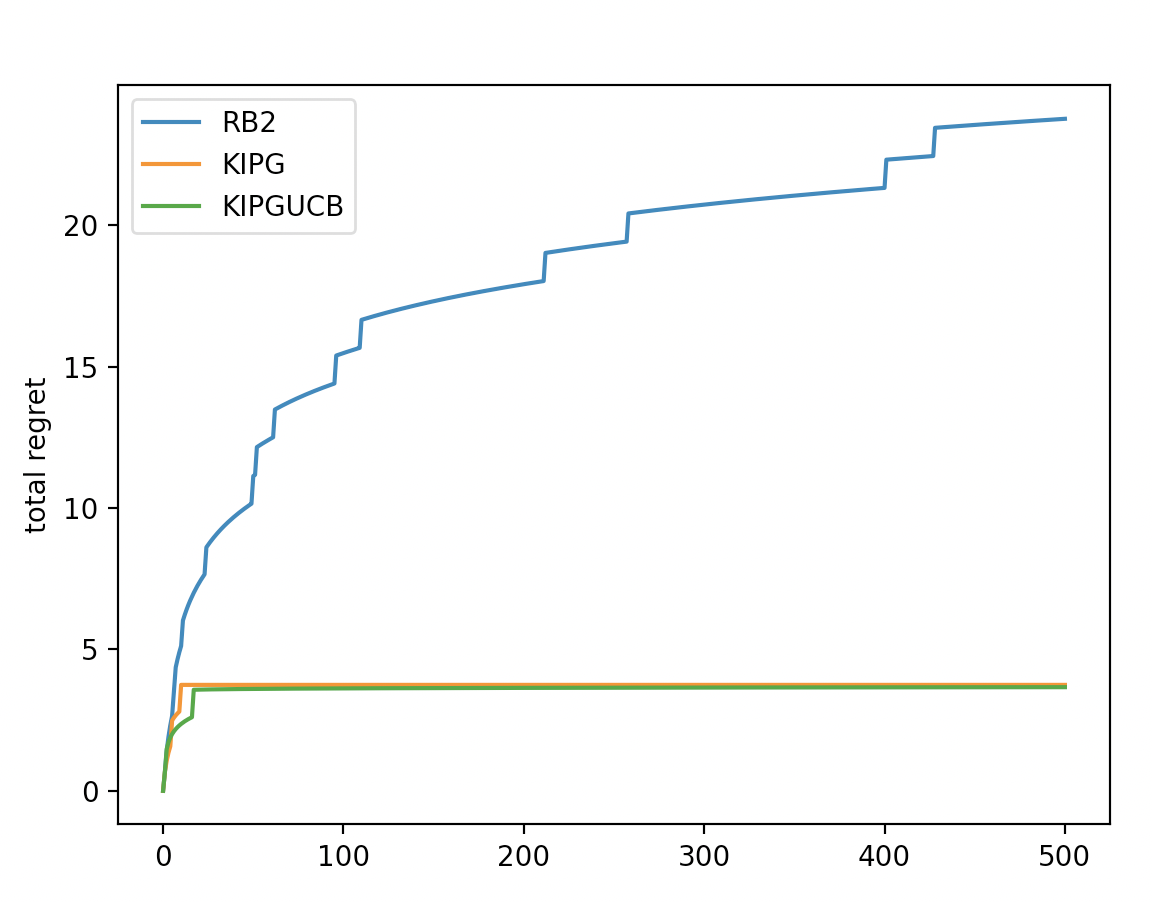}
         \caption{Behavior A}
         \label{fig:behan}
     \end{subfigure}%
     ~
     \begin{subfigure}{0.3\textwidth}
         \centering
         \includegraphics[width=\textwidth]{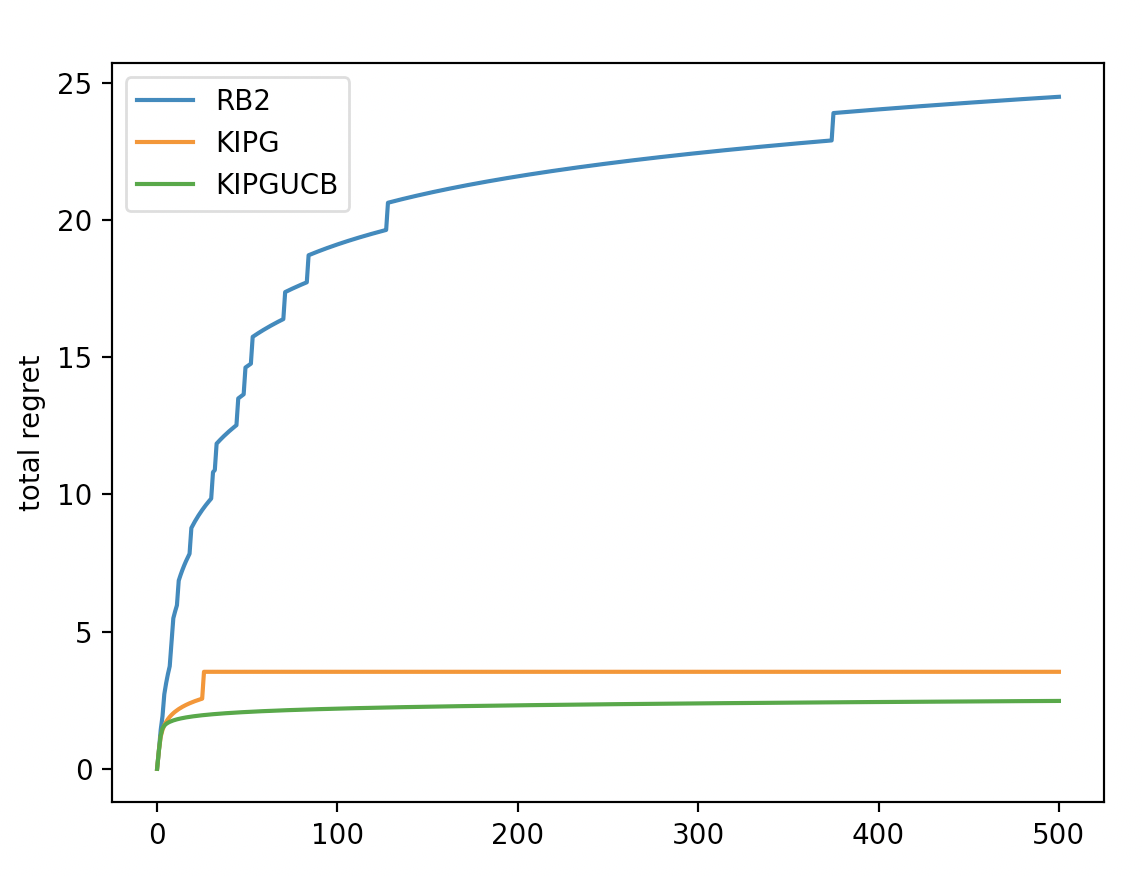}
         \caption{Behavior B}
         \label{fig:behbn}
     \end{subfigure}%
     ~
     \begin{subfigure}{0.3\textwidth}
         \centering
         \includegraphics[width=\textwidth]{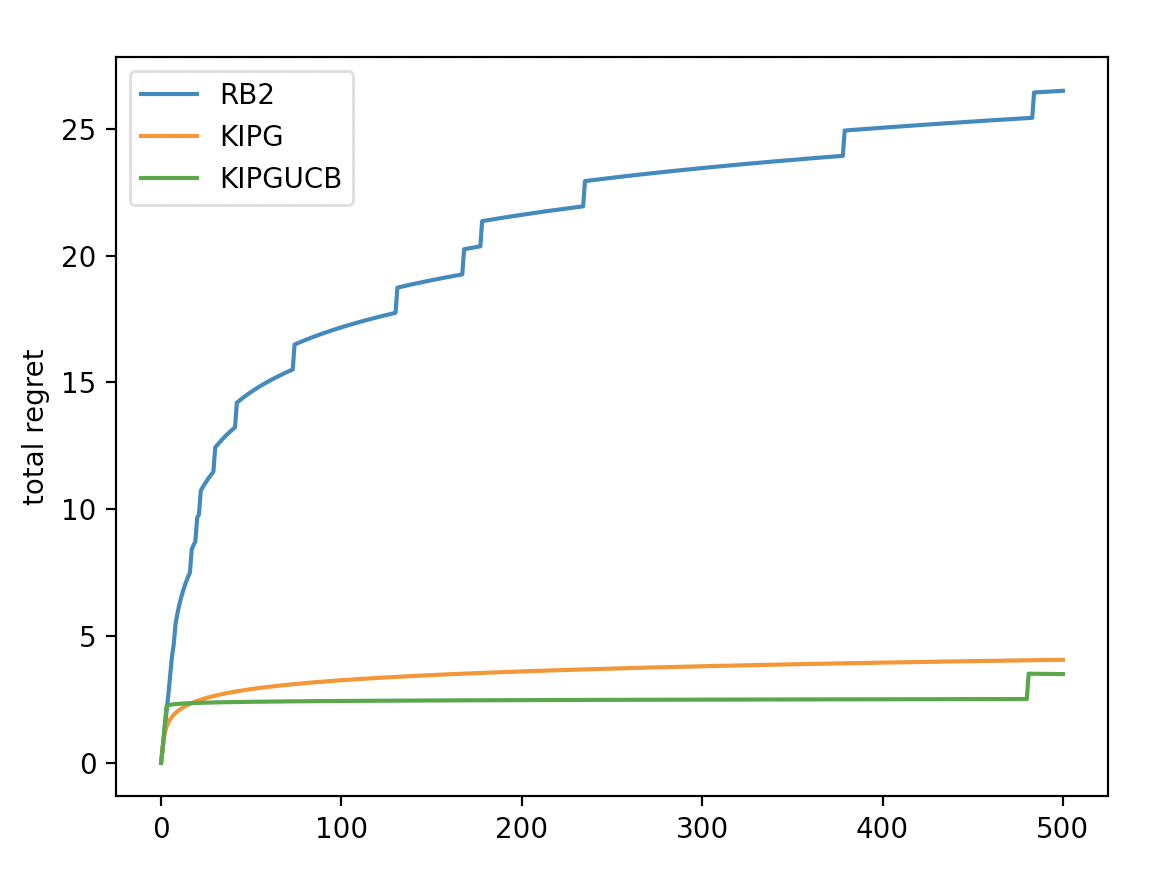}
         \caption{Behavior C}
         \label{fig:behcn}
     \end{subfigure}
    \caption{Shows comparison of $R_{RB2}$, $R_{KIPG}$, $R_{KIPGUCB}$ for the noisy knowledge setting for all three behaviors. As expected we see that $R_{KIPGUCB} < R_{KIPG} < R_{RB2}$ for most $k = 1~to~K$.}
        \label{fig:noisy}
\end{figure*}
\paragraph{Noisy Knowledge: } In this setting the human again observes some user arm interactions to improve the knowledge that they provide. In this case however, the humans observation skills are less sharp. We simulate this scenario by using noisy knowledge for $k=1~to~50$, where perfect knowledge is provided $60\%$ of the time instead of $80\%$. Here, we expect that for most $k=1 to K$, where $K=500$, $R_{KIPGUCB} < R_{KIPG} < R_{RB2}$. We expect this as a perfection rate of $60\%$ means that the tempering of \emph{Knowledge Infusion} by KIPGUCB initially leads to better total regret for KIPGUCB. Figure \ref{fig:noisy} shows this result.
\subsection{Real-World Datasets}
We also evaluate the algorithms in the following real-world datasets:
\begin{itemize}
    \item The Movie Lens dataset with relations such as user age, movietype, movie rating, etc, where the arm label is the genre of a movie. The dataset has $166486$ relational instances \cite{motl2015ctu}.
    \item The Drug-Drug Interaction (DDI) dataset with relations such as Enzyme, Transporter, EnzymeInducer, etc, where the arm label is the interaction between two drugs. The dataset has $1774$ relational instances \cite{dhami2018drug}.
    \item The ICML Co-author dataset with relations such as affiliation, research interests, location, etc, where the arm label represents whether two persons worked together on a paper. The dataset has $1395$ relational instances \cite{dhami2020non}.
    \item The IMDB dataset with relations such as Gender, Genre, Movie, Director, etc, where the arm label is WorkUnder, i.e., if an actor works under a director. The dataset has $938$ relational instances \cite{mihalkova2007bottom}.
    \item The Never Ending Language Learner (NELL) data set with relations such as players, sports, league information, etc, where the arm label represents which specific sport does a particular team plays. The dataset has $7824$ relational instances \cite{mitchell2018never}.
\end{itemize}

\begin{figure*}[!ht]
    \centering
    \includegraphics[width=1.1\textwidth]{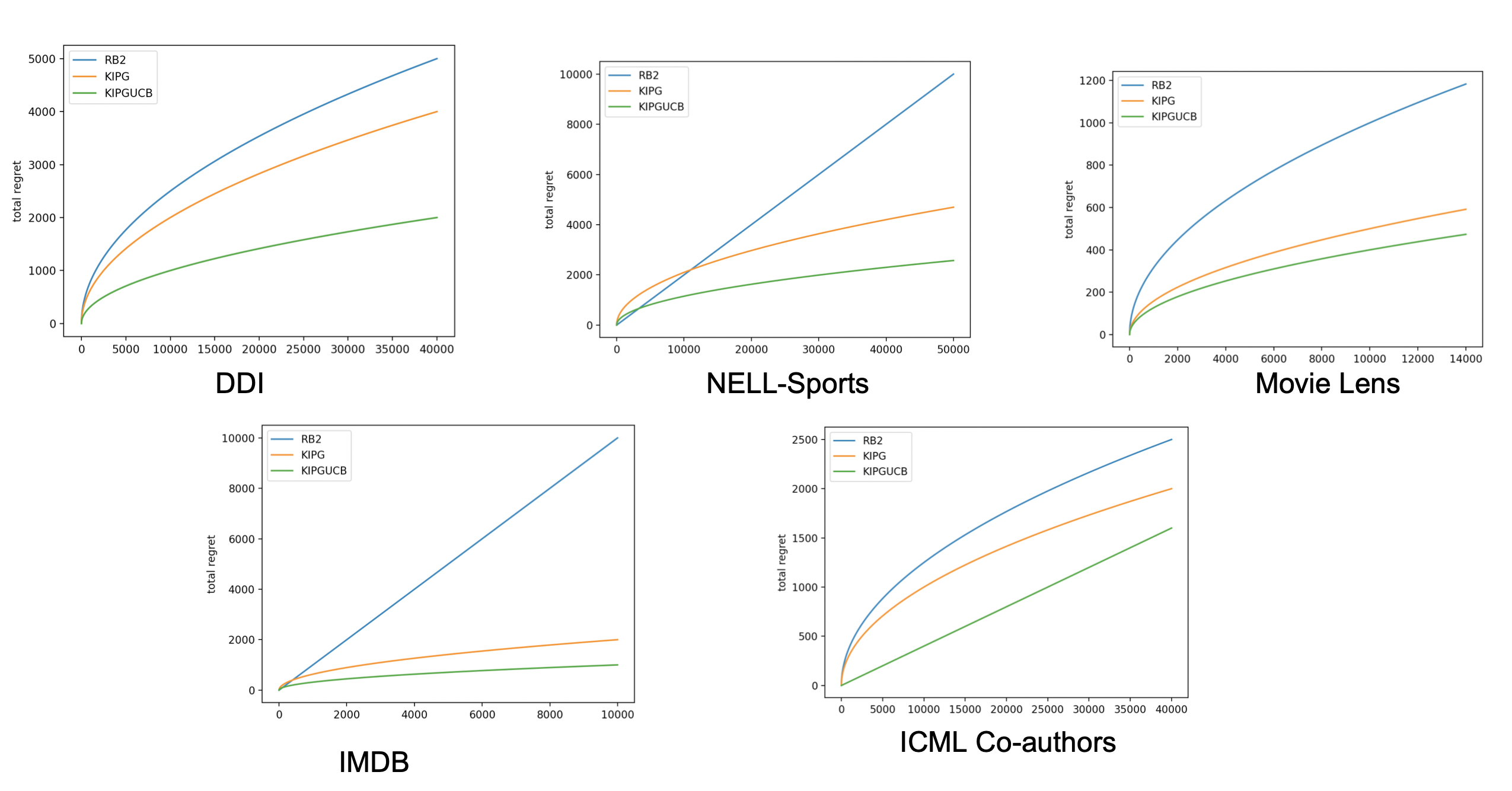}
    \caption{Performance plots computed using total regret of RB2, KIPG, and KIPGUCB for the datasets for $k=1~to~K$. We see that KIPG and KIPG-UCB perform significantly better with expert knowledge in Movie Lens and IMDB compared to others. This is because it is relatively easier for an expert to provide knowledge in these domains.
    On the contrary, 
    in the NELL-Sports, because of noisy knowledge, initially, the performance of KIPGUCB dips compared to RB2, but it increased thereafter.}
    \label{fig:results}
\end{figure*}
 We used $10$ boosted trees for all the experiments and results are averaged over $5$ runs. It is seen that while the total regret remains high for all the datasets over several steps of learning, both the expert knowledge and the exploration strategy using the UCB method are effective in increasing performance. The performance increase is more pronounced in the Movie Lens and IMDB datasets as the expert knowledge are relatively easier to provide for human experts. For the DDI dataset and the ICML Co-authors dataset, it is not straightforward to specify which drugs might interact or which authors may work together in a diverse academic setting. Since the knowledge comes from an expert and systematically targets faster convergence to the optimal distribution, knowledge infusion is expected to perform better. If the knowledge were noisy, the error accumulation over time may have lead to sub-optimal results. In the NELL-sports dataset, it can be seen that RB2 initially outperforms both KIPG and KIPGUCB.
\section{Conclusion and Future Work}
In this study, we develop a novel algorithm KIPGUCB to perform knowledge infusion in CB settings. We show that the regret bound depends on the knowledge and hence the total regret can be reduced if the right knowledge is available. Furthermore, we develop a confidence bound to account for initial uncertainty in provided knowledge in online settings. Though we have developed a general framework for knowledge infusion, we have yet to explore knowledge forms beyond preference knowledge. Furthermore, the knowledge may depend on latent behaviors that cannot be modeled such as a bias by an actor towards a particular director. Also, the actor's bias towards directors may keep changing as more data is seen. This type of non-stationarity and partial observability in context will be interesting to model. Also, if knowledge is noisy and fails to lower total regret, identifying the right descriptive question to ask the human to elicit new knowledge is an interesting future direction. Relational descriptions make tackling this issue plausible. Finally, it will be interesting to mathematically evaluate when the knowledge should be incorporated at all. We aim to tackle these issues in future work.

\bibliographystyle{unsrt}
\bibliography{references}
\end{document}